\documentclass[twoside,11pt]{article}

\usepackage{jmlrwcp2e}
\usepackage{url}
\usepackage{amsmath}
\usepackage{amssymb}
\usepackage{color}
\usepackage{comment}
\usepackage{multirow}
\usepackage{natbib}
\usepackage{graphicx}

\ShortHeadings{Information Criteria and Parameter Shrinkage for Model Selection}{Zhang, Peng, Chan, and Hyv\"{a}rinen} 

\firstpageno{1}
\editor{~}
\begin{document}

\title{Bridging Information Criteria and Parameter Shrinkage\\ for Model Selection}

\author{\name Kun Zhang \email kzhang@tuebingen.mpg.de \\
       \addr Max Planck Institute for Intelligent Systems\\
Spemannstr.\ 38, 72076 T\"ubingen, Germany
       \AND
       \name Heng Peng \email hpeng@math.hkbu.edu.hk \\
       \addr Dept of Mathematics \\
Hong Kong Baptist University, Hong Kong
	\AND
       \name Laiwan Chan \email lwchan@cse.cuhk.edu.hk \\
       \addr Dept of Computer Science and Engineering\\
       Chinese University of Hong Kong
       \AND
       \name Aapo Hyv\"{a}rinen \email aapo.hyvarinen@helsinki.fi \\
       \addr Dept of Computer Science, HIIT,  and Dept of Mathematics and Statistics\\
       University of Helsinki, Finland
     }

\maketitle

\begin{abstract}
Model selection based on classical information criteria, such as
BIC, is generally computationally demanding, but its
properties are well studied. On the other hand,
model selection based on parameter shrinkage by $\ell_1$-type
penalties is computationally efficient. In this paper we make an attempt to
combine their strengths, and propose a 
simple approach that penalizes the likelihood with data-dependent $\ell_1$ penalties as in adaptive Lasso and exploits
a fixed penalization parameter. Even for finite samples, its model selection results
approximately coincide with those based on information criteria; 
in particular, we show that in some special cases, this approach and 
the corresponding information criterion
produce exactly the same model.
One can also consider this approach as a way to directly determine the penalization parameter in adaptive Lasso to achieve information criteria-like model selection.  As extensions, we apply this idea to complex
models including Gaussian mixture model and mixture of factor
analyzers, whose model selection is traditionally difficult to do;
by adopting suitable penalties, we provide continuous
approximators to the corresponding information criteria, which 
are easy to optimize and enable efficient model selection. 

\end{abstract}


\section{Introduction}
Model selection aims at choosing, from a set of candidates, a
mathematical model that strikes a balance between simplicity and
adequacy to the observed data. Traditionally, it is performed by
optimizing some information criteria (ICs). In particular, the Bayesian
information criterion (BIC,~\cite{Schwarz78}),
AIC~\citep{Akaike73}, and the minimum message length (MML)
principle~\citep{Wallace87}, are widely used in different
statistical model selection problems. These criteria have a
discrete feasible domain. 
Their optimization usually involves exhaustive search over all possible models, which is 
computationally intensive.

When the model is very complex or the space of candidate models is
very large, a brute force testing of all possible models causes
very high computational costs and becomes impractical. To tackle
this problem, a lot of efforts have been made to adjust the model
complexity continuously. For instance, for the linear regression problem,
Lasso~\citep{Tibshirani96} applies the $\ell_1$ penalty on the
coefficients which could shrinking some
coefficients to zero. Various approaches, including adaptive Lasso (ALasso,~\cite{Zou06}), SCAD~\citep{Fan01}, and FIRST~\citep{Hwang09} make use of similar but different ways of parameter shrinkage. 
For finite
mixture models, ``entropic
prior"~\citep{Brand99} or the Dirichlet prior~\citep{Zivkovic04} for
the mixing weights could produce sparsity of the mixing weights
and hence perform model selection. However, in these methods, how
to select the penalization parameter 
is usually a crucial issue. Moreover, asymptotic properties of these methods have been well 
studied, but less attention was paid to their performance on finite samples. It would be 
very useful if one can find their relationship to the IC-based approach for finite 
samples.


We aim to develop an efficient model selection approach which is
based on the continuous penalized likelihood and approximately
coincides with model selection based on ICs, such
as BIC. We call this approach quick information criterion-like (Quick-IC) model selection. Our contributions are mainly two fold.  First, for regular models, we establish a bridge between the penalized likelihood of ALasso and ICs, and propose to approximate the latter with the former, resulting in convenient model selection; this can also be considered as a way to directly determine the penalization parameter in ALasso to perform IC-like model selection, which avoids the search for the penalization parameter and would save a lot of computation, especially when iterative procedures are needed to find ALasso solutions.  Specifically, 
 in Sec.~\ref{Sec:ALasso}, we give the intuition that the penalty term for each parameter in ALasso 
is closely related to an indicator function
showing if this parameter is active. Consequently, one can approximate
the number of free parameters in ICs in terms of such penalty terms and
find continuous approximators to the ICs. 
This inspires the proposed approach Quick-IC in Sec.~\ref{Sec:Equivalence}, 
which is shown to select exactly the same model
as the corresponding IC does in the case with a diagonal Fisher information matrix. General cases are also briefly discussed. 
The theoretical claims are verified by simulation studies in Section 4. 

Second, in Sec.~\ref{Sec:make_continuous}, we extend Quick-IC to
non-regular and complex models, such as factor analysis, the Gaussian mixture model
and the mixture of factor analyzers~\citep{Hinton97_MFA}, whose
model selection is traditionally very difficult due to the
large candidate model space. By making use of logarithm penalties
with data-dependent weights, we provide continuous approximators
to the ICs suitable for model selection of these
models, and make their model selection easy and efficient. This illustrates the good applicability of the proposed approach.

\section{Relating Adaptive Lasso to Information Criteria: Intuition} \label{Sec:ALasso}

In this section we assume that the model under consideration satisfies
some regularity conditions including identification conditions for
the parameters $\boldsymbol{\theta}$, the consistency of the maximum likelihood estimate
(MLE) $\hat{\boldsymbol{\theta}}$ when the sample size $n$
tends to infinity, and the asymptotic normality of
$\hat{\boldsymbol{\theta}}$. The penalized likelihood can be
written as
\begin{equation} \label{Eq:VS}
pl(\boldsymbol{\theta})=l(\boldsymbol{\theta})-\lambda
p_\lambda(\boldsymbol{\theta}).
\end{equation}
where $l(\boldsymbol{\theta})$ is the log-likelihood,
$\boldsymbol{\theta}$ is the parameter vector,
$p_\lambda(\boldsymbol{\theta})=\sum_ip_\lambda(\theta_i)$ is the
penalty, and $\lambda$ is the penalization parameter. The maximum
penalized
 likelihood estimate is
$\hat{\boldsymbol{\theta}}_{pl}=\arg\max\limits_{\boldsymbol{\theta}}
pl(\boldsymbol{\theta})$. 

$p_\lambda(\boldsymbol{\theta_i})=\sum_i|\theta_i|$ gives the
$\ell_1$-norm penalty. 
The $\ell_1$ penalty produces sparse and continuous estimates~\citep{Tibshirani96}, and
it has been shown to outperform other penalties in
some scenarios~\citep{Ng04}. However, it also 
causes bias in the estimate of significant parameters, and it
could select the true model consistently only when the data
satisfy certain conditions~\citep{Zhao06}. Certain methods, including stability selection with the 
randomized Lasso~\citep{Meinshausen10} and ALasso~\citep{Zou06},
were proposed to overcome such disadvantages of the $\ell_1$
penalty. 

In particular, ALasso uses $p_\lambda(\boldsymbol{\theta}) =
\sum_{i}\hat{w}_i |\theta_i|$, with
$\hat{{w}}_i=1/|\hat{{\theta}}_i|^\gamma$, where $\gamma>0$, and
$\hat{\boldsymbol{\theta}}$ is a (initial) MLE 
of $\boldsymbol{\theta}$. 
Consequently, the strength for penalizing
different parameters depends on the magnitude
of their estimate. Under some regularity
conditions and the condition $\lambda_n/\sqrt{n}\rightarrow 0$ and
$\lambda_n n^{(\gamma-1)/2}\rightarrow \infty$ (the subscript $n$
is used in $\lambda_n$ to indicate the dependence of $\lambda$ on the sample size
$n$), the ALasso estimator is consistent in model
selection. We are more interested in its behavior on finite samples.

The result of ALasso depends on the penalization parameter
$\lambda$. For very simple models, one may use 
least angle regression (LARS,~\cite{Efron04})
to compute the entire solution path, which gives all possible
solutions as $\lambda$ changes. Among these
solutions, the best model can then be selected by cross-validation
or based on some ICs~\citep{Zou07_df}. (The latter 
approach is compared with our approach in Sec.~\ref{Sec:Compare}, and one can see that it may give very different results from the corresponding IC.) However,
for complex models, especially when iterative algorithms are used
to find the solution corresponding to a given $\lambda$, it is 
computationally very demanding and
impractical to find the solution path. One then needs to select
the penalization parameter in advance. In the next section we show that one can simply determine this parameter, while the model selection
result approximately coincides with that based on ICs.

Let us focus on the case $\gamma=1$, meaning that
\begin{equation} \label{Eq:Adap_lasso_here}
p_\lambda(\theta_i) = \hat{w}_i|\theta_i|=
|\theta_i|/|\hat{\theta}_i|.
\end{equation}
After the convergence of
the ALasso procedure, insignificant parameters become zero, and
$p_\lambda(\theta_i)=0$ for such parameters. On the
other hand, 
with suitable $\lambda_n$, the ALasso estimator
$\hat{\theta}_{i,AL}$ is also consistent~\citep{Zou06}; roughly speaking, significant
parameters are then expected to be changed little by the penalty,
when $n$ is not very small. Consequently, at
convergence, $p_\lambda(\theta_i)=
|\hat{\theta}_{i,AL}|/|\hat{\theta}_i| \approx 1$ for significant
parameters. That is, the penalty $p_\lambda(\theta_i)$ {\it
approximately indicates whether the parameter $\theta_i$ is active or not}.
Suppose that the parameters are not redundant.
$\sum_ip_\lambda(\theta_i)$ is then an approximator of the number
of active parameters, denoted by $D$, in the resulting model.

Recall that the traditional ICs whose
minimization enables model selection can be written as
\begin{equation}\label{Eq:IC}
\textrm{IC}_D = -l(\hat{\boldsymbol{\theta}}_{D}) + \lambda_{IC}D.
\end{equation}
The BIC and AIC criteria are
obtained by setting the value of $\lambda_{IC}$ to \footnote{There exist 
useful modifications of these ICs to accommodate different effects; for instance, as extensions of BIC, 
Draper's IC (DIC,~\cite{Draper95}) and extended BIC (EBIC,~\cite{Chen08_EBIC}) improve the performance of BIC in the small sample size case and in the case with very large model spaces, respectively. However, for simplicity, here we take BIC and AIC, which are widely used, as examples.}
\begin{equation} \label{Eq:lambda_ALasso}
\lambda_{BIC}=1/2\cdot \log n, \textrm{~~and~~} \lambda_{AIC}=1,
\end{equation}
respectively.
 Relating (\ref{Eq:IC}) to the penalized likelihood (\ref{Eq:VS}),
one can see that in ALasso, by setting
 $\lambda = \lambda_{IC}$ ($\lambda_{IC}$ may be
 $\lambda_{BIC}$, $\lambda_{AIC}$, etc.), the maximum penalized
 likelihood is closely related to the IC (\ref{Eq:IC}). This will be rigorously studied next, and in fact {\it $\lambda = 2\lambda_{IC}$} (instead of $\lambda = \lambda_{IC}$) {\it gives interesting results}.

\section{Basic Approach for Quick-IC Model Selection} \label{Sec:Equivalence}

\subsection{With a Diagonal Fisher Information Matrix}
Can we make the model selection results of ALasso exactly the same
as those based on the ICs? In fact, if the Fisher
information matrix is diagonal, this can be achieved
by simply setting $\lambda$ in ALasso to $2\lambda_{IC}$, i.e.,
maximizing the following penalized likelihood
\begin{equation} \label{Eq:proposed_PL}
pl(\boldsymbol{\theta}) = l(\boldsymbol{\theta}) -
2\lambda_{IC}\sum_i|\theta_i|/|\hat{\theta}_i|
\end{equation}
selects the same model as the IC
(\ref{Eq:IC}) does, as seen from the following proposition.

\newtheorem{Prop}{Proposition}
\begin{Prop} \label{Prop1}
Suppose that the following conditions hold.
1. The log-likelihood $l(\boldsymbol{\theta})$ is
quadratic around the MLE
$\hat{\boldsymbol{\theta}}$, with a non-singular observed
Fisher information matrix
$\hat{\mathbf{I}}(\hat{\boldsymbol{\theta}})$. 
2. $\hat{\mathbf{I}}(\hat{\boldsymbol{\theta}})$ is diagonal. 
Then
the non-zero parameters selected by maximizing
(\ref{Eq:proposed_PL}) are exactly those selected by minimizing
the IC (\ref{Eq:IC}).
\end{Prop}

\begin{proof}Since the MLE
$\hat{\boldsymbol{\theta}}$ maximizes $l(\boldsymbol{\theta})$, we
have $\frac{\partial l(\boldsymbol{\theta})}{\partial
\boldsymbol{\theta}}\Big|_{\boldsymbol{\theta} =
\hat{\boldsymbol{\theta}}} = \mathbf{0}$. Let $ \mathbf{H}
\triangleq n\hat{\mathbf{I}}(\hat{\boldsymbol{\theta}})$. Under
the assumptions made in the proposition, the log-likelihood
becomes
\begin{equation} \label{Eq:Like_approx}
l(\boldsymbol{\theta}) =  l(\hat{\boldsymbol{\theta}}) -
\frac{1}{2} (\boldsymbol{\theta} - \hat{\boldsymbol{\theta}})^T
\mathbf{H} (\boldsymbol{\theta} - \hat{\boldsymbol{\theta}}) 
= l(\hat{\boldsymbol{\theta}}) -
\frac{1}{2}\sum_i \mathbf{H}_{ii} ({\theta}_i -
\hat{{\theta}}_i)^2.
\end{equation}
The penalized likelihood (\ref{Eq:proposed_PL}) then becomes
\begin{eqnarray*} 
 pl(\boldsymbol{\theta}) &=& l(\boldsymbol{\theta}) - 
%
%
2 \lambda_{IC}\sum_i(|\theta_i|/|\hat{\theta}_{i}|) \\ 
&=&
l(\hat{\boldsymbol{\theta}}) -
 \frac{1}{2}\sum_{i}\mathbf{H}_{ii}(\theta_i -
 \hat{\theta}_{i})^2 - 2\lambda_{IC}\sum_i(|\theta_i|/|\hat{\theta}_{i}|).
\end{eqnarray*}
It is easy to show that the solutions maximizing
$pl(\boldsymbol{\theta})$ are
\begin{equation} \nonumber
\hat{\theta}_{i,AL} = \textrm{sgn} (\hat{\theta}_{i})\cdot
\big(|\hat{\theta}_{i}| - 2\lambda_{IC} /(\mathbf{H}_{ii}\cdot
|\hat{\theta}_{i}|)\big)^+. 
\end{equation}
That is, $\hat{\theta}_{i,AL}$ estimated by maximizing
(\ref{Eq:proposed_PL}) is non-zero if and only if
$|\hat{\theta}_{i}| - 2\lambda_{IC}\cdot
 /(\mathbf{H}_{ii}\cdot
|\hat{\theta}_{i}|) > 0$, i.e.,
\begin{equation} \label{Eq:criterion_pl}
\mathbf{H}_{ii} \cdot \hat{\theta}_{i}^2 > 2\lambda_{IC}.
\end{equation}

On the other hand, the model selected by minimizing the criterion
(\ref{Eq:IC}) has $D^*$ free parameters if $IC_{D^*} < IC_{{D^*}-1}$ and
$IC_{D^*} \leq IC_{{D^*}+1}$. According to (\ref{Eq:IC}), we then have
\begin{equation} \label{Eq:l_change}
l(\hat{\boldsymbol{\theta}}_{{D^*}+1}) -
l(\hat{\boldsymbol{\theta}}_{{D^*}}) \leq
\lambda_{IC},~\textrm{and~}l(\hat{\boldsymbol{\theta}}_{D^*}) -
l(\hat{\boldsymbol{\theta}}_{D^*-1}) > \lambda_{IC}.
\end{equation}
The least change in $l(\boldsymbol{\theta})$ caused by eliminating
a particular parameter has been derived in the optimal brain
surgeon (OBS) technique~\citep{Hassibi93secondorder}. Here, due to
the simple form of (\ref{Eq:Like_approx}), the least change in
$l(\boldsymbol{\theta})$ caused by eliminating $\theta_i$, denoted
by $S_i$, can be seen directly:
\begin{equation} \label{Eq:S_i}
S_i = \frac{1}{2}\hat{\theta}_i^2/[\mathbf{H}^{-1}]_{ii}  =
\frac{1}{2} [\mathbf{H}]_{ii} \cdot \hat{\theta}_i^2.
\end{equation}
Note that $l(\hat{\boldsymbol{\theta}}_{D^*}) -
l(\hat{\boldsymbol{\theta}}_{D^*-1})$ in (\ref{Eq:l_change}) is
the minimum of $S_i$ for all $D^*$ parameters in the current model.
Therefore, one can see that model selection based on the
IC (\ref{Eq:IC}) selects $\theta_i$ if and
only if $S_i > \lambda_{IC}$, which is equivalent to the
constraint (\ref{Eq:criterion_pl}).
That is, under the assumptions made in the proposition, non-zero
parameters produced by maximizing the penalized likelihood
(\ref{Eq:proposed_PL}) are exactly those selected by the
corresponding IC (\ref{Eq:IC}). $\blacksquare$ \end{proof}

 This proposition indicates that (\ref{Eq:proposed_PL}) can be considered as a continuous
approximator to the ICs (\ref{Eq:IC}), which enables Quick-IC; one can
see that the continuous approximator is obtained by simply
replacing the maximum likelihood $l(\hat{\boldsymbol{\theta}}_D)$ by the data
likelihood $l(\boldsymbol{\theta})$, and replacing the number of free parameters, $D$, by
$2\sum_i|\theta_i|/|\hat{\theta}_i|$.

\subsection{More General Case}
The condition in Proposition~\ref{Prop1} is rather restrictive; in the 
linear regression scenario, it corresponds to the orthogonal design case. In the more 
general case, where $\hat{\mathbf{I}}(\hat{\boldsymbol{\theta}})$ is usually not
diagonal, 
the condition for the parameters $\theta_i$ to be
selected by maximizing (\ref{Eq:proposed_PL}) becomes more complex, and (\ref{Eq:IC}) and
(\ref{Eq:proposed_PL}) are usually not exactly equivalent. We give some results on the relationship between Quick-IC 
and model selection based on ICs.

\begin{Prop} \label{Prop2}
Suppose that condition 1 in Proposition~\ref{Prop1} holds. Assume that both the IC approach (\ref{Eq:IC}) and Quick-IC 
(\ref{Eq:proposed_PL}) perform model selection in the backword elimination manner, 
i.e., the penalization parameter is gradually 
increased to the target value, such that insignificant parameters are set to zero one by one. 
Further assume that once a parameter is set to zero, it will not become non-zero again. Let $\widetilde{\mathbf{H}} \triangleq \textrm{diag}(\hat{\boldsymbol{\theta}})\cdot
\mathbf{H}\cdot\textrm{diag}(\hat{\boldsymbol{\theta}})$, where $\mathbf{H} = n\hat{\mathbf{I}}(\boldsymbol{\theta})$ and is assumed to be nonsingular.
Then the IC approach (\ref{Eq:IC}) selects
$\theta_i$  if and only if $[\widetilde{\mathbf{H}}^{-1}]_{ii} < \frac{1}{2\lambda_{IC}}$, while 
Quick-IC (\ref{Eq:proposed_PL}) does so if and only if $[\widetilde{\mathbf{H}}^{-1}]_{i\cdot}\mathbb{I} <
\frac{1}{2\lambda_{IC}}$, where $[\widetilde{\mathbf{H}}^{-1}]_{i\cdot}$ donotes the $i$th row of 
$\widetilde{\mathbf{H}}^{-1}$ and $\mathbb{I}$ is the vector of 1's.
\end{Prop}

\begin{proof} From the proof of Proposition~\ref{Prop1} or~\cite{Hassibi93secondorder}, one can see that the IC
approach selects $\theta_i$ if and only if 
$S_i = \frac{1}{2}\hat{\theta}_i^2/{[\mathbf{H}^{-1}]_{ii}} > \lambda_{IC}$, which is equivalent to 
$[\widetilde{\mathbf{H}}^{-1}]_{ii} <
\frac{1}{2\lambda_{IC}}$. Let $\tilde{{\theta}}_i \triangleq {{\theta}}_i/\hat{{\theta}}_i $. 
On 
the other hand,  due to Condition 1 in Proposition~\ref{Prop1}, the penalized likelihood with the penalization 
parameter $\lambda$ is
\begin{eqnarray} \nonumber
J &=& l(\boldsymbol{\theta})-\lambda\sum_i|\theta_i|/|\hat{\theta}_i| \\ \nonumber
&=& l(\hat{\boldsymbol{\theta}}) - 
\frac{1}{2} (\boldsymbol{\theta} - \hat{\boldsymbol{\theta}})^T
\mathbf{H} (\boldsymbol{\theta} - \hat{\boldsymbol{\theta}}) -\lambda\sum_i|\theta_i|/|\hat{\theta}_i| \\ \nonumber
&=&  l(\hat{\boldsymbol{\theta}}) - \frac{1}{2} (\tilde{\boldsymbol{\theta}} - \mathbb{I})^T
\widetilde{\mathbf{H}} (\tilde{\boldsymbol{\theta}} - \mathbb{I}) - \lambda \sum_i \tilde{\theta}_i.
\end{eqnarray}
Clearly, if $\lambda$ is very small such that none of $\theta_i$ 
is set to zero, $J$
 is maximized when $\frac{\partial J}{\partial \tilde{\boldsymbol{\theta}}} = \mathbf{0}$, i.e., 
$$\widetilde{\mathbf{H}}(\mathbb{I} - \tilde{\boldsymbol{\theta}})
-\lambda \mathbb{I} = \mathbf{0},$$ 
which is equivalent to $\mathbb{I} - \tilde{\boldsymbol{\theta}} = 
\lambda \widetilde{\mathbf{H}}^{-1}\mathbb{I}$. Consequently, we have 
$$\tilde{\theta}_i = 1 - \lambda [\widetilde{\mathbf{H}}^{-1}]_{i\cdot}\mathbb{I}.$$ 
When $\lambda$ is gradually increased such that 
$\lambda [\widetilde{\mathbf{H}}^{-1}]_{j.}\mathbb{I} = 1$, $\tilde{\theta}_j$, 
or equivalently ${\theta}_j$, is set to zero. Finally, when $\lambda$ is increased to $2\lambda_{IC}$, 
the non-zero parameters $\theta_i$ selected by Quick-IC satisfy $[\widetilde{\mathbf{H}}^{-1}]_{i\cdot}\mathbb{I} < \frac{1}{2\lambda_{IC}}$. $\blacksquare$
\end{proof}

Although in practice one may not adopt backword elimination, the above proposition helps us understand the similarity and difference between 
the IC approach and Quick-IC. For example, if $\sum_{j\neq i} [{\mathbf{H}}^{-1}]_{ij} = 0$ for all $i$ (which includes 
Proposition~\ref{Prop1} as a special case), the two approaches
give the same results. Of course, for finite samples, in the general case it is theoretically impossible to make 
parameter shrinkage-based 
Quick-IC exactly identical to the IC approach. However, their empirical comparisons in various situations presented 
in Sec.~\ref{Sec:Compare} suggest that they usually give the
same model selection results for various sample sizes. 
We give the following remarks on the proposed model selection
approach. Firstly, the result of the proposed approach depends on
$\hat{\boldsymbol{\theta}}$. When the model is very large,
$\hat{\boldsymbol{\theta}}$ may be too rough, and it is useful to
update $\hat{\boldsymbol{\theta}}$ using a consistent estimator
sometime when a smaller model is derived.\footnote{Note that this
is different from the reweighted $\ell_1$ minimization methods (see, e.g.,~\cite{Candas08}). In the reweighted
methods, in each iteration the penalized estimate given in the
previous iteration is used to form the new weight; in this way,
the reweighted ALasso penalty provides an approximator to the
logarithm penalty, since $\log(\theta)$ can be locally
approximated by $|\theta|/|\theta_0|$ plus some constant, about
point $\theta_0$.} Secondly, in Sec.~\ref{Sec:make_continuous} the idea of Quick-IC is further applied to more complex 
models, 
by using data-dependent weights
for suitable penalization functions and approximating the number of effective parameters. For example,
in some cases one needs to resort to the logarithm penalty to
produce sparsity of parameters, and we suggest using the
corresponding data-adaptive penalty
$\hat{w}_i\log(\frac{|\theta_i|+\epsilon}{\epsilon})$ with
$\hat{w}_i = 1/\log(\frac{|\hat{\theta}_i|+\epsilon}{\epsilon})$,
where $\epsilon$ is a very small positive number, as the penalty
term, as discussed in Section~\ref{Sec:make_continuous_MML}.
Correspondingly, to obtain the continuous approximator of the
ICs, one just simply replaces the number of
effective parameters in $\theta_i$ with $2 \sum_i
\hat{w}_i\log(\frac{|\theta_i|+\epsilon}{\epsilon})$.

\section{Numerical Studies}\label{Sec:Compare}

The proposed approach in Section~\ref{Sec:Equivalence} directly
applies to model selection of simple models such as regression and
vector auto-regression (VAR). VAR provides a convenient way for
Granger causality analysis~\citep{Granger80}, and has a lot of
applications in economics, neuroscience, etc. Unfortunately it
usually involves quite a large number of parameters, making
the IC approach impractical, while Quick-IC 
gives efficient model selection.

In this section we use simulations to investigate the performance of Quick-IC. To 
verify the results in Sec.~\ref{Sec:Equivalence}, we consider the simple linear regression problem 
$y = \boldsymbol{\theta}^T \mathbf{x} + \epsilon$, where $y$ is the target, $\mathbf{x} = (x_1,...,x_p)^T$ contains predictors, 
and $\epsilon$ is the Gaussian noise.
We take BIC as an example, i.e., we compare BIC-like Quick-IC 
(or Quick-BIC, with $\lambda_{IC} = \lambda_{BIC}$ in (\ref{Eq:proposed_PL})) with the original BIC (\ref{Eq:IC}). 
We also compare them with the approach of ALasso followed by the BIC criterion (ALasso+BIC): 
one first finds the solution path of ALasso using LARS, and then selects the ``best" model by evaluating the BIC criterion with the maximum 
likelihood $l(\hat{\boldsymbol{\theta}}_D)$ replaced by the likelihood of the parameter values on the solution 
path (\cite{Zou07_df}, Sec. 4). For this reason, ALasso+BIC is different from BIC. In Quick-BIC, the noise 
variance was estimated from the full model. For BIC, we searched the prediction number between 4 and 8.

20 predictors $x_i$ were used, i.e., $p=20$. 14 entries of $\boldsymbol{\theta}$ were set to zero. The magnitudes of the others were randomly 
chosen between 0.2 and 2.5, and the signs were arbitrary. We considered three cases. Case I corresponded to an 
orthogonal design, i.e., all predictors are uncorrelated. In Case II, the pairwise correlation between $x_i$ 
and $x_j$ was set to be $0.5^{|i-j|}$. In the last case, the covariance matrix of $\mathbf{x}$ was randomly generated 
as $\mathbf{M}\mathbf{M}^T$ with entries of the square matrix $\mathbf{M}$ randomly sampled between $-0.5$ and $0.5$. In all cases we 
normalized the variance of each $x_i$. The noise variance was $0.5$. To see the sample size effect, we varied the sample size $n$ from 
100 to 300. The simulation was repeated for 100 random trials.

Table~\ref{tab1} reports the frequency of the differences in the selected predictor numbers given by different 
methods. One can see that in Case I, all the three methods almost always select the same number of 
predictors. In Cases II and III, Quick-BIC still gives rather similar results to BIC; in particular, as 
the sample size increases, their results tend to agree with each other quickly.  
ALasso+BIC 
produces different models with a surprisingly noteworthy chance for both sample sizes, especially in Case III. 
However, it seems to be still statistically consistent in model selection, like BIC; we found that when 
$n=600$, for 56 times it gave the same model as BIC. As for the computational loads, BIC took more than 
550 times longer than Quick-BIC as well as ALasso+BIC.


\begin{small}\begin{table*}[htbp]
\centering
\caption{Occurrence frequency of the differences in the numbers of selected predictors by different methods for 100 random trials. ``$<$" and ``$>$" mean ``$<-1$" and ``$>1$", respectively; $D_{BIC}^*$, $D_{Quick-BIC}^*$, and $D_{ALasso+BIC}^*$ denote 
the numbers of predictors produced by BIC, Quick-BIC, and ALasso+BIC, respectively.} \vspace{0.2cm}
\label{tab1}
\begin{tabular}{|c|c||c|c|c|c|c||c|c|c|c|c||c|c|c|c|c|}
\hline
\multirow{2}{0.5cm}{\small Case} & \multirow{2}{0.25cm}{\small$T$} & \multicolumn{5}{c}{\small$D_{BIC}^*-D_{Quick-BIC}^*$} \vline\vline   
& \multicolumn{5}{c}{\small$D_{BIC}^*-D_{ALasso+BIC}^*$} \vline\vline    & \multicolumn{5}{c}{\begin{small}$D_{Quick-BIC}^* - D_{ALasso+BIC}^*$\end{small}} \vline  \\
  \cline{3-17}
  &  & \small-2 & \small-1 & \small \bf0 & \small1 & \small2  & \small$<$ &\small-1& \small \bf0 &\small1 &\small$>$  &\small$<$& \small-1 &\small \bf0 &\small1 &\small$>$\\
   \hline \hline
\multirow{2}{0.5cm}{\small I}  & \small100  & \small0 & \small4 & \small \bf96 & \small0 & \small0 & \small0& \small3 &   \small \bf97 & \small0 & \small0 & \small0& \small0& \small \bf99 & \small1& \small0\\
 &\small300  & \small0 & \small1& \small \bf99 & \small0 & \small0 & \small0& \small0& \small \bf100 &\small0& \small0&  \small0& \small0& \small \bf99 & \small1& \small0\\ \hline \hline
\multirow{2}{0.5cm}{\small II} & \small100  & \small1 & \small5 & \small \bf92 & \small2 & \small0 & \small2& \small3&    \small \bf73 & \small15& \small7 & \small2& \small1& \small \bf74& \small16& \small7\\
 &\small300  & \small0 & \small0& \small \bf100 &\small0 & \small0 & \small0 & \small0& \small \bf89 &\small11& \small0 & \small0& \small0& \small \bf89 & \small11& \small0\\ \hline \hline
\multirow{2}{0.5cm}{\small III}& \small100  & \small3 & \small12  &\small \bf67 & \small14& \small4 & \small20& \small29& \small \bf33 & \small13& \small5 & \small21 & \small26& \small \bf37 &\small11& \small5\\
 &\small300  & \small0 & \small6& \small \bf87 & \small7 & \small0 & \small 27& \small29& \small \bf42 &\small2& \small0 & \small 27& \small32 & \small \bf38& \small3 & \small0\\
\hline
\end{tabular}
\end{table*}\end{small}


\section{Extensions: Quick-IC by Approximating Various Information Criteria}\label{Sec:make_continuous}

Below we focus on other frequently-used statistical models,
especially some complex ones, and give continuous approximators to
the ICs for their model selection by extending Quick-IC. 
We also give empirical results 
to illustrate the applicability and efficiency of Quick-IC.

\subsection{General Framework with Grouped Parameters}

For regular statistical models, under a set of regularity
conditions, the asymptotic normality of $\hat{\boldsymbol\theta}$
holds. The $\ell_1$ penalty used in Lasso can then produce
sparsity of the parameters and hence perform model
selection~\cite{Tibshirani96}. 
 The
asymptotic properties of the variable selection techniques
established in the linear regression scenario also approximately
hold for regular models. For some non-regular models, it is still
possible to do so. If the gradient of the log-likelihood changes
slowly around $\hat{\boldsymbol\theta}$, these penalties will
successfully push insignificant parameters to zero. Otherwise, one
may apply penalization on suitable transformations of the
parameters, instead of the original parameters.

In practice, 
the parameters in a model often naturally belong to groups, 
i.e., they are selected or discarded 
simultaneously~\citep{Yuan06_grouped,Bach08}. One can
formalize this by introducing functions $\mathcal{T}_i$ which
allow computation of the penalties for groups of variables.
Generally speaking, the information criterion of the form
(\ref{Eq:IC}) can be approximated by the negative penalized
likelihood:
\begin{equation} \label{Eq:npl}
npl(\boldsymbol{\theta}) =- l(\boldsymbol{\theta}) +
2\lambda_{IC}\cdot\sum_i D_{fi}\cdot
p_\lambda(\mathcal{T}_i(\boldsymbol{\theta})),
\end{equation}
where $\mathcal{T}_i(\boldsymbol{\theta})$ are suitable
transformations of the parameters (or selected parameters)
controlling the complexity of the model, and $D_{fi}$ are the
numbers of independent parameters associated with the group
$\mathcal{T}_i(\boldsymbol{\theta})$. Minimization of the negative
penalized likelihood (\ref{Eq:npl}) enables simultaneous model
selection and parameter estimation. When a particular
$\mathcal{T}_i(\boldsymbol{\theta})$ is pushed to zero, $D_{fi}$
free parameters disappear, and the model
complexity is reduced. How to choose
$\mathcal{T}_i(\boldsymbol{\theta})$ and to calculate $D_{fi}$
depends on the specific model.

%
%

\subsection{Quick BIC-Like Model Selection for Factor Analysis} \label{Sec:FA}

Let us first consider model selection of the factor analysis (FA)
model. In FA, the observed $d$-dimensional data vector
$\mathbf{x}=(x_1,\cdots,x_d)^T$ is modeled as $
\mathbf{x} = \mathbf{Ay} + \mathbf{e}$,
where $\mathbf{A}$ is the factor loading matrix, $\mathbf{y} =
(y_1,\cdots,y_k)^T$ the vector of $k$ underlying Gaussian factors,
and $\mathbf{e}=(e_1,\cdots,e_d)^T$ the vector of uncorrelated
Gaussian errors with the covariance matrix
$\boldsymbol{\Psi}=\textrm{diag} (\psi_1,\cdots,\psi_d)$. The
factors $\mathbf{y}$ and the errors $\mathbf{e}$ are also mutually
independent. Here, we have assumed that $\mathbf{x}$ is zero-mean
and that the factors $\mathbf{y}$ are normally distributed with
zero mean and identity covariance matrix.

Given the factor number $k$ and a set of observations
$\{\mathbf{x}_t\}_{t=1}^n$, the FA model can be fitted by
maximum likelihood (ML) using the expectation-maximization (EM)
algorithm~\citep{Rubin82,Ghahramani97_MFA}. But ML estimation
could not determine the optimal factor number $k^*$, since the
ML does not consider the complexity of the model and it
increases as $k$ grows. 

A suitable factor number gives the FA model enough capacity and
avoids over-fitting. When the unconditional variances of 
$y_i$ are fixed, model selection of FA can be achieved by
shrinking suitable columns of $\mathbf{A}$ to
zero. So entries in each column of $\mathbf{A}$ are grouped. 
Denote by $\mathbf{A}_{(i)}$ the $i$th column of $\mathbf{A}$.
Note that $|| \mathbf{A}_{(i)} ||$ is singular at
$\mathbf{A}_{(i)} = \mathbf{0}$, so penalization on $||
\mathbf{A}_{(i)} ||$ can remove unnecessary columns in
$\mathbf{A}$ and consequently perform model selection. The
negative penalized likelihood for approximating BIC is
\begin{equation} \label{Eq:PLL_FA} npl^{FA} =
-l^{FA}  + 2D_{f}\cdot \lambda_{BIC}
\sum_{i=1}^kp_\lambda(||\mathbf{A}_{(i)}||),
\end{equation} where $D_{f}$ denotes the number
of free parameters in the column of $\mathbf{A}$ which is to be
removed, and $\lambda_{BIC}$ is
given in (\ref{Eq:lambda_ALasso}). Due to the
rotation indeterminacies of the factors $y_i$, the total number
of free parameters in
$\mathbf{A}$ is $dk - \frac{k(k-1)}{2}$. 
The proposed method removes columns of $\mathbf{A}$ one by one.
If one insignificant column of $\mathbf{A}$ is shrinked to zero,
the total number of free parameters in $\mathbf{A}$ reduces from
$dk - \frac{k(k-1)}{2}$ to $d(k-1) - \frac{(k-1)(k-2)}{2}$.
Therefore, $D_f$ can be
evaluated to equal 
$d-k+1$, as the change of the number of free parameters in
$\mathbf{A}$ when a certain column disappears. Once a column of
$\mathbf{A}$ is removed, $k$ is updated accordingly.

The EM algorithm for minimizing the negative penalized likelihood
(\ref{Eq:PLL_FA}) can be derived analogously to the derivation
of that for the FA model~\citep{Ghahramani97_MFA}.
Following~\cite{Fan01}, we use the local quadratic approximation
(LQA) to approximate the penalties
$p_\lambda(||\mathbf{A}_{(i)}||)$. As a great advantage, it admits
a closed-form solution for $\mathbf{A}$ in the M step. 

We would like to address the following advantages of adopting the
negative penalized likelihood based on ALasso, instead of the
original BIC criterion, for model selection. 1. The negative
penalized likelihood is easy to minimize. 2. If the log-likelihood
function is concave in the neighborhood of the maximum likelihood
estimator (like in the linear regression problem), the negative
penalized likelihood is {\it convex}, and its minimization does
not suffer from multiple local minima.


%

%

\subsection{Quick MML-Like Model Selection for Gaussian Mixture Model} \label{Sec:make_continuous_MML}
The Gaussian mixture model (GMM) models the density of the
$d$-dimensional variable $\mathbf{x}$ as a weighted sum of some
Gaussian densities:
$f(\mathbf{x}) = \sum_{i=1}^m\pi_i
\phi(\mathbf{x};\boldsymbol{\mu}_i, \boldsymbol\Sigma_i)$,
where $\phi(\mathbf{x};\boldsymbol{\mu}_i, \boldsymbol\Sigma_i)$ are
Gaussian densities with mean $\boldsymbol{\mu}_i$ and covariance
matrix $\boldsymbol\Sigma_i$, and $\pi_i$ are nonnegative weights
that sum to one.

BIC is not suitable for model selection of mixture models, since
not all data are effective for estimating the
parameters specifying an individual component. 
Instead, the MML-based model selection criterion is
preferred~\citep{Figueiredo02}. The message length to be minimized
for model selection of GMM is
\begin{eqnarray} \label{Eq:MML_GMM}
\mathcal{L}^{GM}_{MML} = \frac{D_f^{GM}}{2}\sum_{i:\pi_i>0}
\log\Big(\frac{n\pi_i}{12}\Big)+
\frac{m_{nz}}{2}\log\Big(\frac{n}{12}\Big) 
+ \frac{m_{nz}(D_f^{GM}+1)}{2} - l^{GM},
\end{eqnarray}
where $m_{nz}$ denotes the number of non-zero-probability
components, and the number of free parameters in each component is
$D_f^{GM}=d+d(d+1)/2 = d^2/2+3d/2$. Minimization of the above
function is troublesome since it involves the discrete variable
$m_{nz}$. Below we develop an approximator to (\ref{Eq:MML_GMM})
which is continuous in $\pi_i$.

GMM is a typical non-regular statistical model. The expected
complete-data log likelihood of GMM (see~\cite{McLachlan00} for
its formulation), which gives an approximation of the true data
likelihood, involves $\log \pi_i$. Hence, its gradient w.r.t.
$\pi_i$ grows very fast when $\pi_i \rightarrow 0$. Consequently,
the $\ell_1$ penalty could not push insignificant $\pi_i$ to zero.
Fortunately, one can then naturally exploit the $\log$ penalty to produce sparsity
of $\pi_i$. The
$\log$ penalty on $\pi_i$ also has the advantage of admitting a
closed-form update equation for $\pi_i$ in the EM algorithm. To
avoid the discontinuity of the objective function when a component
with $\pi_i\rightarrow 0$ vanishes, we use
$\log(\frac{\varepsilon+\pi_i}{\varepsilon}) = \log(\varepsilon +
\pi_i) - \log(\varepsilon)$ as the penalty, where $\varepsilon$ is a
small enough positive number (we chose $10^{-3}$ in experiments).
Let $\hat{w}_i \triangleq 1/\log[(\varepsilon +
\hat{\pi}_i)/\varepsilon]$. Inspired by the idea of adaptive weights
in ALasso, we can let the penalty term be
$\hat{w}_i\log[(\pi_i+\epsilon)/\epsilon]$. $m_{nz}$ could then be
approximately by 2$\sum_{i=1}^m \hat{w}_i{\log[(\varepsilon +
\pi_i)/ \varepsilon]}$. Consequently, (\ref{Eq:MML_GMM}) is
approximated by
\begin{eqnarray} \label{Eq:MML_GMM_continuous} 
npl^{GM}_{MML} = \frac{D_f^{GM}}{2}\sum_{i:\pi_i>0}
\log\Big(\frac{n\pi_i}{12}\Big) +\Big[ \log\Big(\frac{n}{12}\Big)
+ D_f^{GM}+1
\Big]\sum_{i=1}^m\hat{w}_i\log\Big(\frac{\pi_i+\epsilon}{\epsilon}
\Big) - l^{GM}.
\end{eqnarray}

The EM algorithm for minimizing the function above is the same as
that for maximizing the GMM likelihood, except that the update
equation for $\pi_i$ is changed to
\begin{eqnarray} \nonumber
\pi_i = \max\Big\{0, \frac{ \sum_{t=1}^n h_{it} -
0.5D_f^{GM} - 0.5\hat{w}_i  [\log(n/12) +
D_f^{GM}+1]}{n-0.5mD_f^{GM}- 0.5[ \log(n/12) +
D_f^{GM}+1]\sum_{j=1}^m\hat{w}_j} \Big\},
\end{eqnarray}
where $h_{it}$ denotes the posterior probability that the $t$th
point comes from the $i$th component. When $\pi_i$ becomes very
small, say smaller than $1/n$, we drop the $i$-th component. In
practice, if the initialized model is very far from the desired
one, as the model complexity reduces, it is better to occasionally
update $\hat{w}_i$ with the corresponding maximum likelihood
estimator.

\subsection{Quick MML-Like Model Selection for Mixture of Factor Analyzers}

Now consider the mixture of factor analyzers
(MFA, \cite{Hinton97_MFA}), which has a lot of applications in pattern recognition. 
 It assumes that the 
$d$-dimensional observations $\mathbf{x}$ can be modeled as
$\mathbf{x} = \boldsymbol{\mu}_i + \mathbf{A}_i\mathbf{y}_i +
\mathbf{e}_i\textrm{~with
probability~}\pi_i\textrm{~}(i=1,\cdots, m)$,
where $\boldsymbol{\mu}_i$ is the mean of the $i$th factor
analyzer, and local factors $\mathbf{y}_i$, which follow $
\mathcal{N}(\mathbf{0},\mathbf{I}_{k_i})$, are independent from
$\mathbf{e}_i$, which follow
$\mathcal{N}(\mathbf{0},\boldsymbol{\Psi})$ with
$\boldsymbol{\Psi}=\textrm{diag} (\psi_i,\cdots,\psi_d)$. The
factor number $k_i$ may vary for different $i$.

Following~\cite{Figueiredo02}, one can find the message length
for MFA (with some constant terms dropped):
\begin{eqnarray} \label{Eq:MML_MFA} 
\mathcal{L}^{MFA}_{MML} &=& \frac{1}{2}\sum_{i:\pi_i>0}\Big[
D_{fi}^{F}\cdot \log\Big(\frac{n\pi_i}{12}\Big)\Big] + 
\frac{1}{2}\sum_{i=1}^{m_{nz}}D_{fi}^{F}  
+
\frac{m_{nz}}{2}\Big[\log\Big(\frac{n}{12}\Big)+1\Big] - l^{MFA},
\end{eqnarray}
where $D_{fi}^{F}$ denotes the number of free parameters specifying
the $i$-th factor analyzer, i.e., $D_{fi}^{F} = d
+dk_i-k_i(k_i-1)/2$. This function involves integers $m_{nz}$ (the
number of factor analyzers) and $k_i$, $i=1,\cdots,m_{nz}$, (the
number of factors in each factor analyzer). Its optimization is
computationally highly demanding due to the large search space of
$(m_{nz}, {k_i}_{i=1}^{m_{nz}} )$. 
 Using the
$\ell_1$ and $\log$ penalties with data-adaptive weights, we can
approximate $\mathcal{L}^{MFA}_{MML}$ with the following function:
\begin{eqnarray} \nonumber  npl_{MML}^{MFA} &=&
\frac{1}{2}\sum_{i:\pi_i>0}\hat{D}_{fi}^F\log\Big(\frac{n\pi_i}{12}
\Big) +
\sum_{i=1}^m\hat{w}_i\log\Big(\frac{\pi_i+\epsilon}{\epsilon}\Big)
\\ \label{Eq:npl_MML_MFA} &&\cdot \Big[ \hat{D}_{fi}^F + \log\Big( \frac{n}{12} \Big)+1  \Big]
-l^{MFA},
\end{eqnarray}
where $\hat{D}_{fi}^F = \hat{N}_f(\mathbf{A}_i) + d$ and
$\hat{N}_f(\mathbf{A}_i)$ is the ALasso-based approximator to the
number of free parameters in $\mathbf{A}_i$. After some
derivations, one can see that a reasonable approximator is
$\hat{N}_f(\mathbf{A}_i) = \big[dk_i-\frac{k_i(k_i-1)}{2}\big] -
D_f \sum_{j=1}^{k_i} \Big(1 - p_\lambda(||\mathbf{A}_{i,(j)}||)
\Big) = \frac{k_i(k_i-1)}{2} + (d-k_i+1)\sum_{j=1}^{k_i}
p_\lambda(||\mathbf{A}_{i,(j)}||)$. One can verify that
$\hat{N}_f(\mathbf{A}_i)$ changes very slightly when $k_i$ is
reduced by shrinking columns of $\mathbf{A}_{i}$.
Similar to~\cite{Ghahramani97_MFA}, one can derive the EM
algorithm for minimizing (\ref{Eq:npl_MML_MFA}).

We note that the proposed model selection methods for GMM and MFA
generate new components or split any large component. For very
complex problems, they may converge to local optima. If necessary,
one can perform the split and merge operations~\citep{Ueda99} after
certain EM iterations to improve the final results.
\subsection{Experiments}
\subsubsection{Factor Analysis}
We generated the data according to the FA model, and compared four
model selection schemes, which are BIC-like Quick-IC given in Section~\ref{Sec:FA} (or Quick-BIC),
BIC, AIC, and fivefold cross-validation (CV). The true factor
number was $k = 5$, and the data dimension was $d=10$.
Elements in $\mathbf{A}$ were randomly
generated between $-1.5$ and $1.5$, and the error variances
$\psi_i$ were random numbers between 0 and 1. When using BIC, AIC,
or CV, we let $k_{min}=3$ and $k_{max} = 7$, while Quick-BIC was
initialized with 8 factors. To investigate the sample size effect,
we let the sample size $n$ be 40 and $100$. In each case, we
repeated all methods for 100 random trials.

When $n=40$, BIC, AIC, Quick-BIC, and CV approximately took 4, 4,
1.5, and 20 seconds, respectively, for each trial. Clearly
Quick-BIC is most computationally appealing, as expected.
Fig.~\ref{fig:FA_EXP} plots the histogram of the factor numbers
found by the four methods. When $n=40$, BIC (as well as
Quick-BIC) seems to over-penalize the model complexity and
results in a smaller factor
number. 
But when $n$ is increased to 100, its performance becomes almost
the best. On the contrary, AIC seems to under-penalize the
complexity. In both cases, Quick-BIC is always similar to BIC. Also considering its light
computational load, Quick-BIC is preferred among the four
methods. We also tested the case $n=200$, and found that
Quick-BIC and BIC give clearly the best results.

\begin{figure}[htp]
\centering
\includegraphics[width=5.2in,height=1.0in]{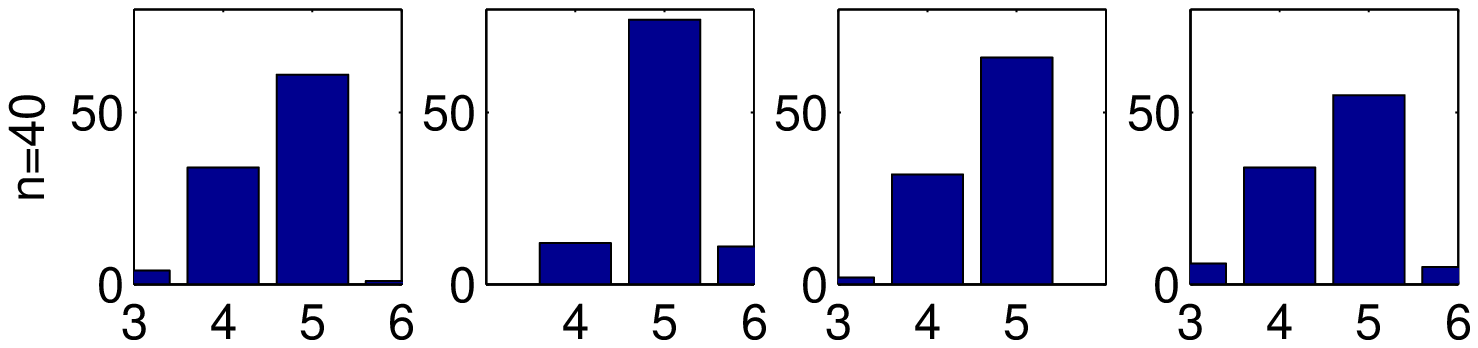}\\
~~~~~(a) BIC~~~~~~~~~~(b) AIC~~~~~~~~(c) Quick-BIC~~~~~(d) CV~~~\vspace{0.1cm} \\
\includegraphics[width=5.2in,height=1.0in]{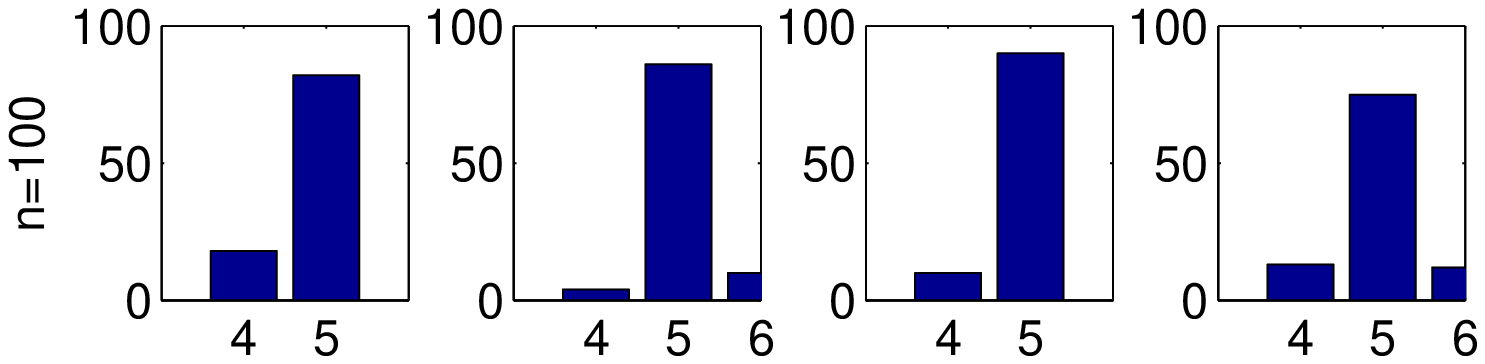}\\
~~~~~(e) BIC~~~~~~~~~~(f) AIC~~~~~~~~(g) Quick-BIC~~~~~(h) CV~~~\\
 \caption{Occurrence frequency for various $k$ (100 trials). The true value is 5. Top/bottom: $n$=40/100. } \label{fig:FA_EXP}
\end{figure}

\subsubsection{Gaussian Mixture Model}
We compared the approach Quick-IC which minimizes the
continuous version of MML (\ref{Eq:MML_GMM_continuous}) with
the MML-based method proposed
in~\cite{Figueiredo02} (denoted
by FJ's method), in terms of the chances of finding the preferred
component number and
the CPU time. 


Here we present the results on two data sets. For each data set,
we repeated each method for 100 trials. In each trial,
the data were randomly generated, and for initialization, the mean
of each Gaussian component was randomly chosen from the data
points. The results on the ``shrinking spiral" data
set~\citep{Ueda99}  are given in Fig.~\ref{fig:GMM_EXP}(a-c), and
Fig.~\ref{fig:GMM_EXP}(d-f) shows the results on the ``triangle
data", which were obtained by rotating and shifting three sets
of bivariate Gaussian points following
$\mathcal{N}(\mathbf{0},\textrm{diag}(2.25,0.25))$. For the first
data set, we set $m_{max}=30$ for both methods and $m_{min}=3$ for
FJ's method. The CPU time taken by Quick-IC and FJ's method
was $4$ and $33$ seconds, respectively. For the second data set,
we let $m_{max}=20$ and $m_{min} = 1$ for FJ's method. The CPU
time was about $7$ and $21$ seconds for the two methods.
Fig.~\ref{fig:GMM_EXP}(b, e) and (c, f) give the histograms of the
component numbers obtained by FJ's method and Quick-IC. One can
see that they give similar results. However, FJ'
method seems to produce less robust (more disperse) results for
the spiral data. We conjecture that it is caused by the
``annihilation" process in FJ's method~\citep{Figueiredo02}: FJ's
method annihilates the least probable component (with the smallest
mixing weight $\hat{\pi}_i$) to obtain a smaller model. This
process is discontinuous, and simply uses the magnitude of the
mixing weight to indicate the significance of the corresponding
component. In fact the significance of a particular component also
depends on its relationship to other components. 
As a consequence, when a component that has the least weight but
is actually significant is removed, the message length
$\mathcal{L}^{GM}_{MML}$ may increase, resulting in a sub-optimal
model .
\vspace{-0.3cm}
\begin{figure}[htp]
\centering
\hspace{-0.055in}\includegraphics[width=1.5in,height=1.26in]{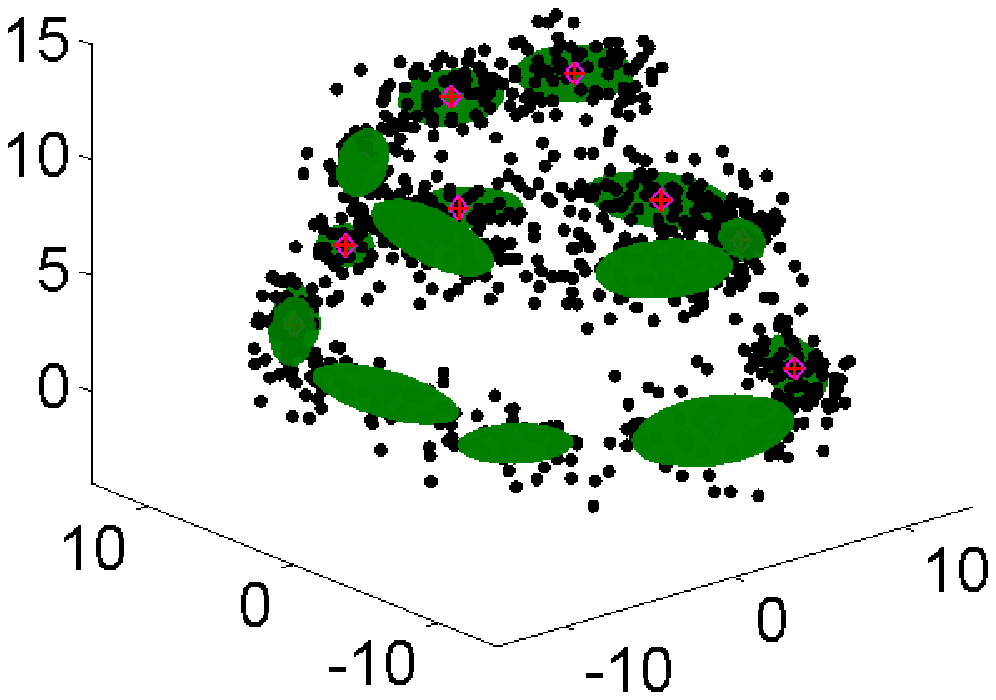}\hspace{-0.055in}\includegraphics[width=2.52in,height=1.26in]{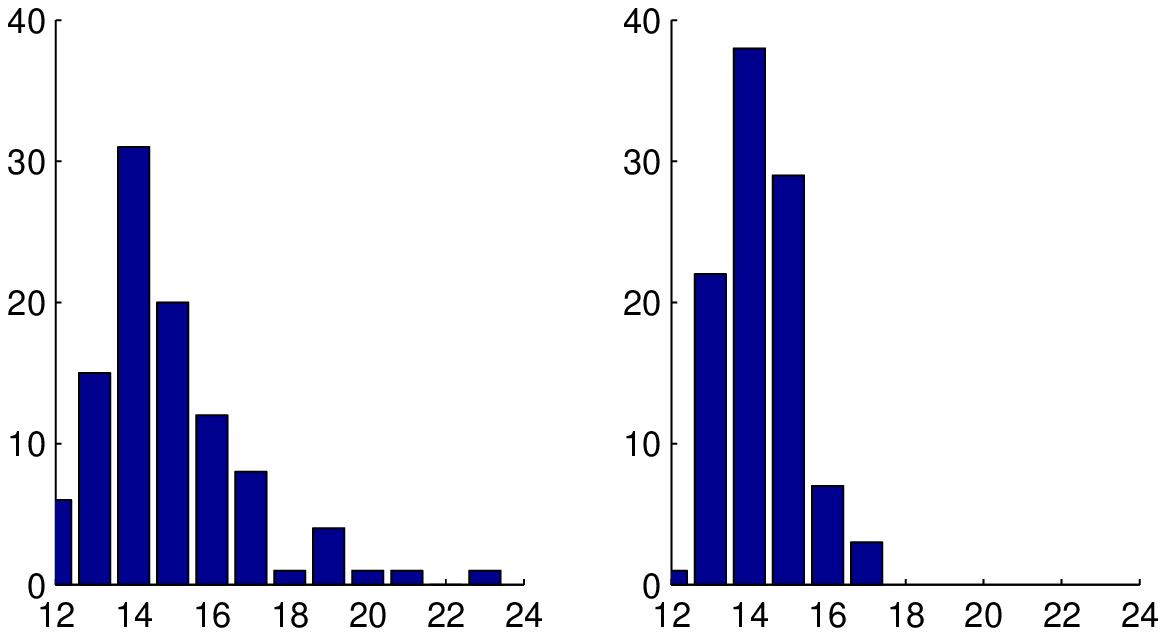}\\
~~~(a) spiral data~~~~~~~~~~~~~~(b) FJ's~~~~~~~~~~(c)
Quick-IC~~~~~~~\\
\hspace{-0.06in}\includegraphics[width=1.5in,height=1.26in]{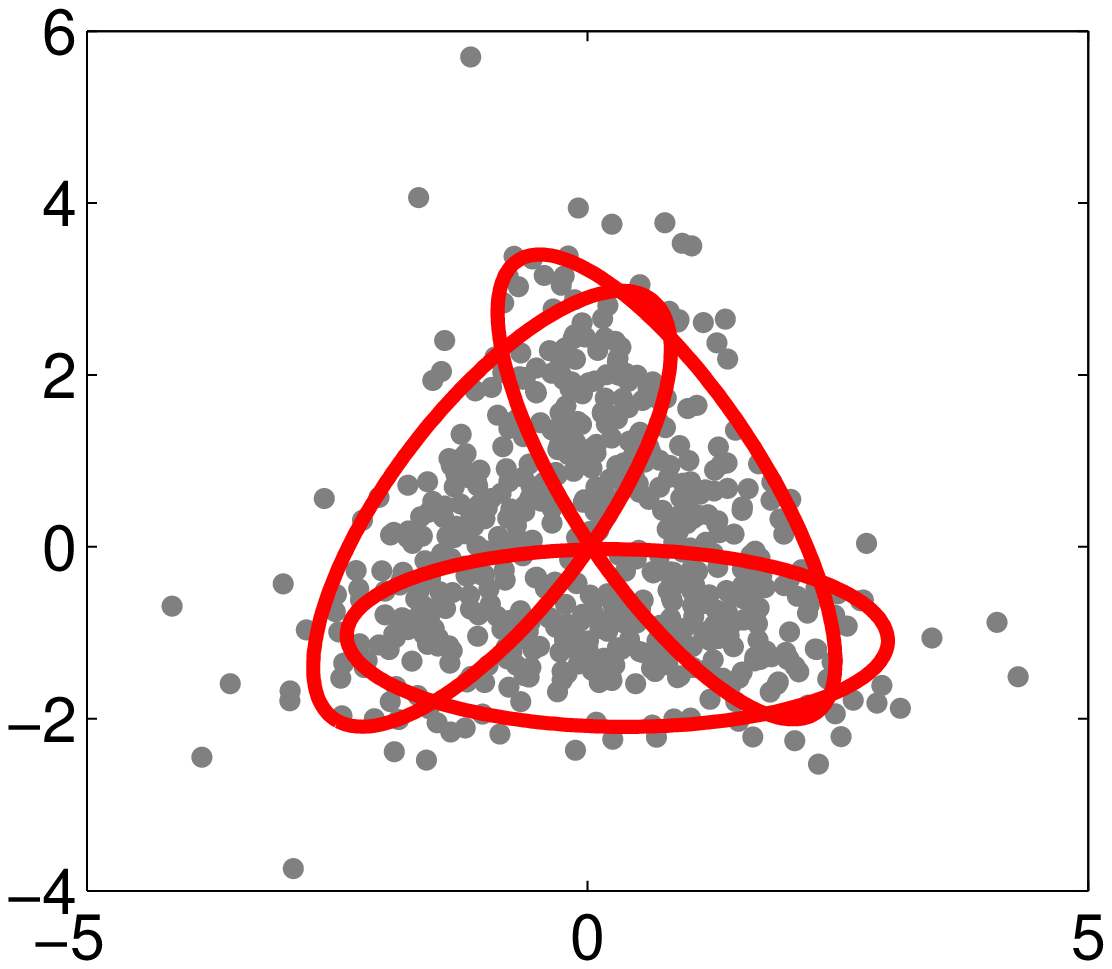}\hspace{-0.055in}\includegraphics[width=2.52in,height=1.26in]{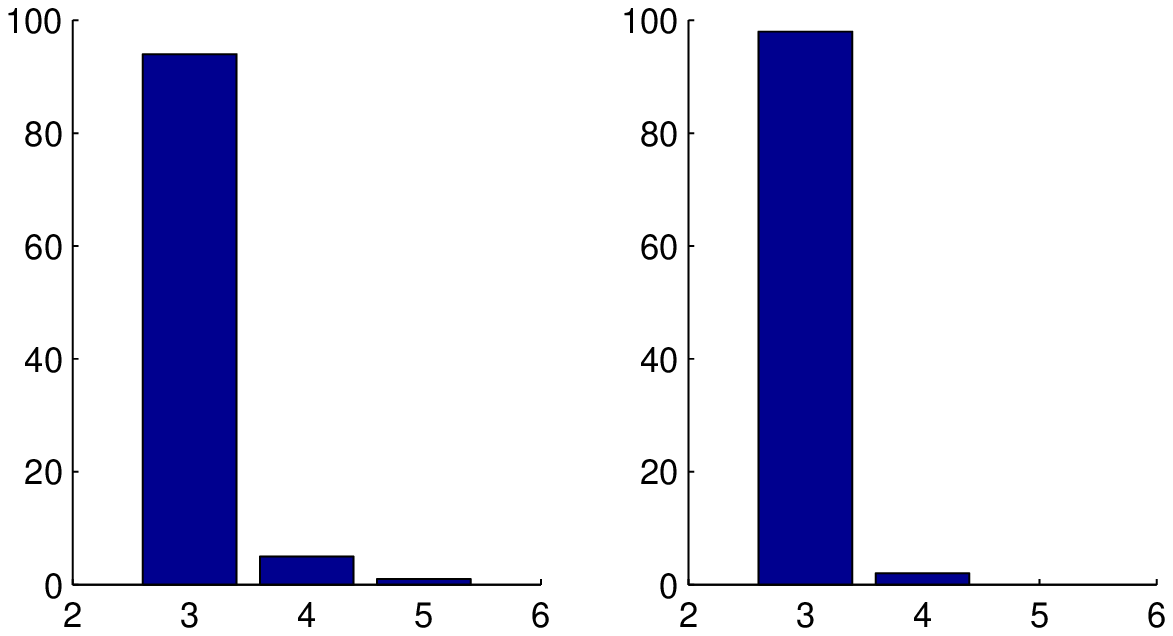}\\
~~~(d) triangle data~~~~~~~~~~~~(e) FJ's~~~~~~~~~~(f) Quick-IC~~~~~~~\\
 \caption{Experimental result on GMM learning. (a-c) spiral data with 900 points. (d-f) triangle
 data with 600 points. (a, d) data points and typical results.
 (b, e) histograms of
the number of components learned by FJ's method for 100
trials. (c, f) histograms of the number of components
learned by Quick-IC.
  } \label{fig:GMM_EXP}
\end{figure}

%

\subsubsection{Mixture of Factor Analyzers}
Quick-IC uses the MML approximator (\ref{Eq:npl_MML_MFA}) to determine
both the number of factor analyzers ($m$) and the local factor
numbers ($k_i$) in the MFA model. We tested the spiral data
(Figure~\ref{fig:GMM_EXP}a), and repeated the experiments with 50
trials. $m$ was initialized as $30$ and all of $k_i$ were
initialized as $3$. The number of factor analyzers learned by our
approach is always between 10 and 13 (with the chances 10: 8/50,
11: 21/50, 12: 14/50, and 13: 7/50). In the resulting model,
most factor analyzers have 1 factor, and occasionally there is one
factor analyzer with 2 factors (with one dominating the
other) or with no factor. This is consistent with the previous
results with $k_i$ {\it a prior} set to 1~\citep{Figueiredo02}.


We then constructed another synthetic data set in which the local factor
number varies for different factor analyzers.
Fig.~\ref{fig:mfa_guy}(a) plots the data points without noise, and
(b) shows the observed noisy data. The sample size was 5390. Quick-IC 
was compared with the variational Bayesian
method (VBMFA,~\cite{Ghahramani00}). 
We repeated both methods for 20 trials with different
initializations. Quick-IC and VBMFA took about $1.5$ and $10$
minutes for each run, and produced $9\sim 13$ and $10\sim14$
factor analyzers, respectively. Note that the data are clearly
non-Gaussian, so some factor analyzers may overlap to some extent
to model the data
well. 
Fig.~\ref{fig:mfa_guy}(c) and (d) show the results of the two
method in one run. Since Quick-IC does not generate new local
factor analyzers, it cannot separate two factor analyzers which
are initialized together. This can be alleviated by using a large
$m$ for initialization. On the other hand, sometimes VBMFA may
split one factor analyzer into two; we found that in 2 trials
VBMFA divided the ``arm'' or ``leg'' into two segments.
\vspace{-0.2cm}
\begin{figure}[htp]\vspace{-.0cm} \centering
\includegraphics[width=3.12in]{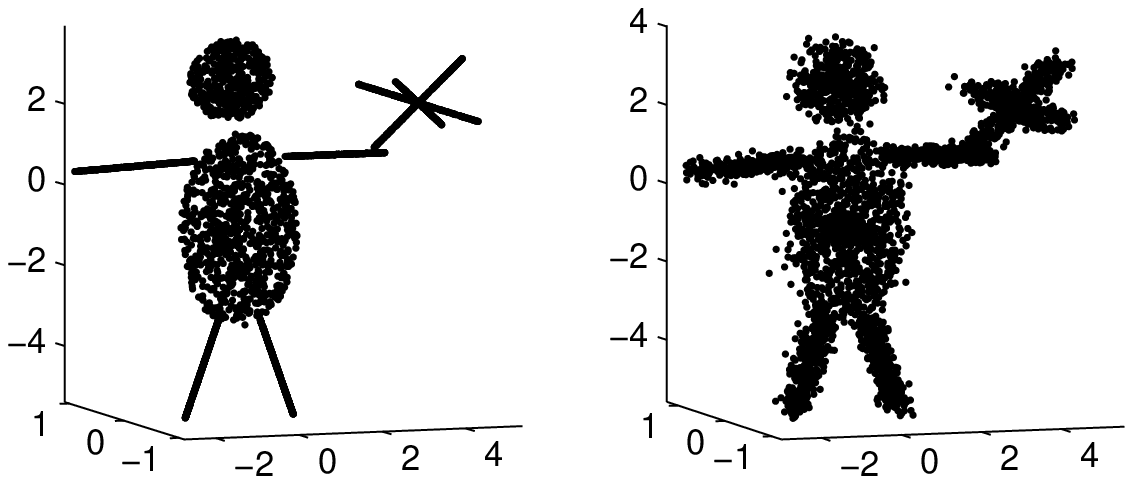}\\
~~~~~~~~~~~~~~~~~~~~~~~~~(a)~~~~~~~~~~~~~~~~~~~~~~~~~~~~~(b) ~~~~~~~~~~~~~~~~~~~~~~\\
\includegraphics[width=3.0in]{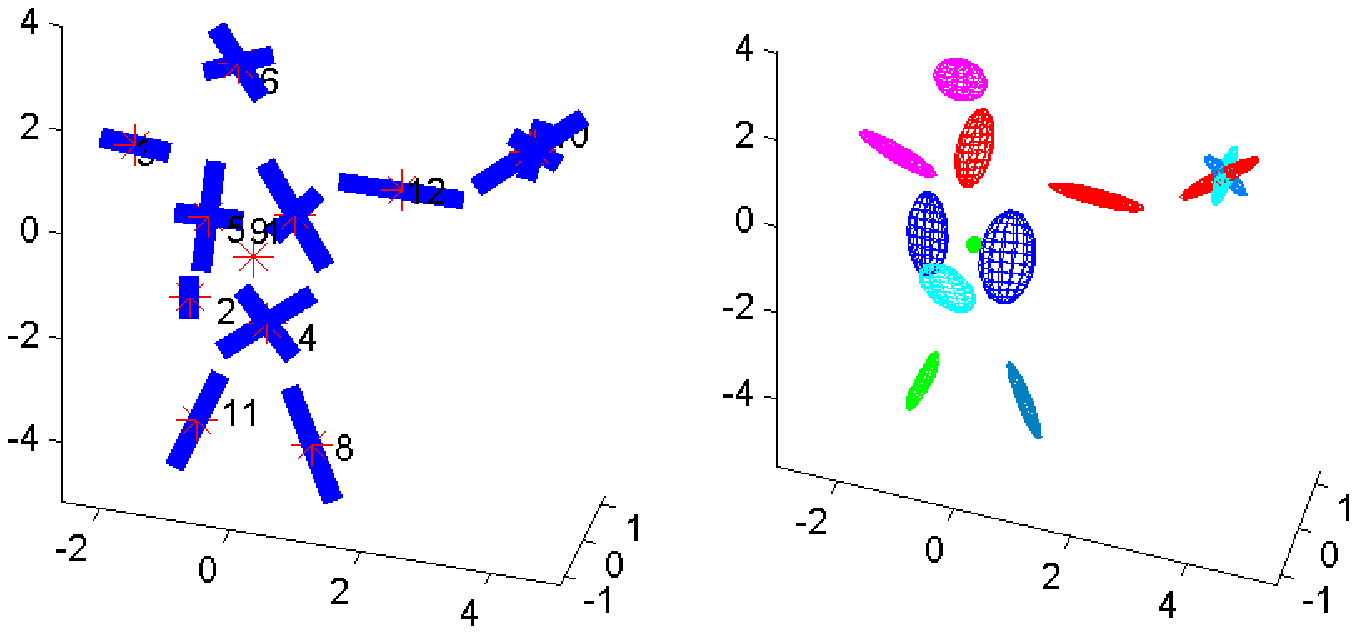}\\ \vspace{-0.2cm}
~~~~~~~~~~~~~~~~~~~~~~~~~(c)~~~~~~~~~~~~~~~~~~~~~~~~~~~~~(d) ~~~~~~~~~~~~~~~~~~~~~~\\
 \caption{Results of MFA with factor analyzers having different
   $k_i$. (a) noiseless points.
 (b) noisy data for analysis.  (c) a structure learned by Quick-IC, 
with $m=12$. (d) that by VBMFA, with $m=12$.
 Note that (c) plots columns of loading matrices, while (d) depicts the contour of the Gaussian distribution of each factor analyzer. } \label{fig:mfa_guy}
\end{figure}

\section{Conclusion and Future Work}
We showed that under some conditions, the penalty used in adaptive
Lasso, which is the $\ell_1$ penalty with a data-dependent weight,
resembles an indicator function showing if this parameter
is active. This motivated us to approximate the
traditional model selection criterion by the penalized likelihood
with a fixed penalization parameter. The latter is continuous in the parameters, 
greatly facilitating the model selection procedure. We formulated this idea as the Quick-IC approach. 
We showed that for finite samples, Quick-IC produces exactly the same
model as the corresponding information criterion when the Fisher information matrix is diagonal. We also investigated more general cases. Furthermore, for
some complex and non-regular models, we provided continuous
approximators to their model selection criteria, by using suitable
penalty forms
and data-adaptive weights. 
We have demonstrated that for these models, our simple
approach is computationally very efficient in model selection, and that its results are similar to those produced by the corresponding IC. One line
of our future research is to investigate the theoretical properties 
of Quick-IC for non-regular models such as finite mixture models.

\end{document}